\documentclass{article}

\PassOptionsToPackage{numbers}{natbib}
\usepackage[final]{neurips_2019}

\usepackage[utf8]{inputenc} 
\usepackage[T1]{fontenc}    
\usepackage{hyperref}       
\usepackage{url}            
\usepackage{booktabs}       
\usepackage{amsfonts}       
\usepackage{amsmath}
\usepackage{amssymb}
\usepackage{amsthm}
\usepackage{nicefrac}       
\usepackage{microtype}      
\usepackage{xcolor}
\usepackage{xspace}
\usepackage{acronym}
\usepackage{bm}
\usepackage{graphicx}
\usepackage{mathtools}
\usepackage{enumitem}
\usepackage{tabularx}
\usepackage{float}

\acrodef{fgsm}[FGSM]{fast gradient sign method}
\acrodef{pgd}[PGD]{projected gradient descent}
\acrodef{rgb}[RGB]{red, green, blue}
\acrodef{hsv}[HSV]{hue, saturation, value}
\acrodef{lpips}[LPIPS]{Learned Perceptual Image Patch Similarity}

\newcommand{\modelname}{ReColorAdv\xspace}
\newcommand{\tmname}{functional\xspace}
\newcommand{\Tmname}{Functional\xspace}
\newcommand{\tm}{{func}}
\newcommand{\tmfunc}{f}

\newcommand{\example}{\mathbf{x}}
\newcommand{\advexample}{\widetilde{\mathbf{x}}}
\newcommand{\clf}{g}
\newcommand{\funcfam}[1]{\mathcal{F}_\text{#1}}

\DeclareMathOperator*{\argmin}{arg\,min}

\newtheorem{theorem}{Theorem}
\newcommand{\smallparagraph}[1]{\textbf{#1}\quad}
\newcolumntype{Y}{>{\centering\arraybackslash}X}
\newcolumntype{b}{>{\centering\arraybackslash}X}
\newcolumntype{s}{>{\hsize=.5\hsize\centering\arraybackslash}X}

\title{\Tmname Adversarial Attacks}

\author{%
  Cassidy Laidlaw \\
  University of Maryland \\
  \texttt{claidlaw@umd.edu} \\
  \And
  Soheil Feizi \\
  University of Maryland \\
  \texttt{sfeizi@cs.umd.edu} \\
}

\begin{document}

\maketitle

\begin{abstract}

We propose \textit{\tmname adversarial attacks}, a novel class of threat models for crafting adversarial examples to fool machine learning models. Unlike a standard $\ell_p$-ball threat model, a \tmname adversarial threat model allows only a \textit{single} function to be used to perturb input features to produce an adversarial example. For example, a \tmname adversarial attack applied on colors of an image can change \textit{all} red pixels simultaneously to light red. Such global uniform changes in images can be less perceptible than perturbing pixels of the image individually. For simplicity, we refer to \tmname adversarial attacks on image colors as \modelname, which is the main focus of our experiments. We show that \tmname threat models can be combined with existing additive ($\ell_p$) threat models to generate stronger threat models that allow both small, individual perturbations and large, uniform changes to an input. Moreover, we prove that such combinations encompass perturbations that would not be allowed in either constituent threat model. In practice, \modelname can significantly reduce the accuracy of a ResNet-32 trained on CIFAR-10. Furthermore, to the best of our knowledge, combining \modelname with other attacks leads to the strongest existing attack even after adversarial training.


\end{abstract}

\section{Introduction}

There is an extensive recent literature on \textit{adversarial examples}, small perturbations to inputs of machine learning algorithms that cause the algorithms to report an erroneous output, e.g. the incorrect label for a classifier. Adversarial examples present serious security challenges for real-world systems like self-driving cars, since a change in the environment that is not noticeable to a human may cause unexpected, unwanted, or dangerous behavior. Many methods of generating adversarial examples (called \textit{adversarial attacks}) have been proposed \citep{szegedy_intriguing_2014,goodfellow_explaining_2015,moosavi-dezfooli_deepfool:_2016,papernot_limitations_2016,carlini_towards_2017}. Defenses against such attacks have also been explored \citep{papernot_distillation_2016,madry_towards_2018,zhang_theoretically_2019}. 

Most existing attack and defense methods consider a threat model of adversarial attacks where adversarial examples can differ from normal inputs by a small $\ell_p$ distance. However, using this threat model that encompasses a simple definition of "small perturbation" misses other types of perturbations that may also be imperceptible to humans. For instance, small spatial perturbations have been used to generate adversarial examples \citep{engstrom_rotation_2017,xiao_spatially_2018,wong_wasserstein_2019}.

In this paper, we propose a new class of threat models for adversarial attacks, called \textit{\tmname threat models}. Under a \tmname threat model, adversarial examples can be generated from a regular input to a classifier by applying a \textit{single} function to all features of the input:
\begin{equation*}
\begin{aligned}
    \text{Additive threat model: \quad}
    & (x_1, \dots, x_n) & \rightarrow \quad
    &  (x_1 + \delta_1, \dots, x_n + \delta_n) \\
    \text{{\bf \Tmname threat model:} \quad}
    & (x_1, \dots, x_n) & \rightarrow \quad
    & (f(x_1), \dots, f(x_n)) \\
\end{aligned}
\end{equation*}
For instance, the perturbation function $\tmfunc(\cdot)$ could darken every red pixel in an image, or increase the volume of every timestep in an audio sample. \Tmname threat models are in some ways more restrictive because features cannot be perturbed individually. However, the uniformity of the perturbation in a \tmname threat model makes the change less perceptible, allowing for larger absolute modifications. For example, one could darken or lighten an entire image by quite a bit without the change becoming noticeable. This stands in contrast to separate changes to each pixel, which must be smaller to avoid becoming perceptible. We discuss various regularizations that can be applied to the perturbation function $\tmfunc(\cdot)$ to ensure that even large changes are imperceptible.

The advantages and disadvantages of additive ($\ell_p$) and \tmname threat models complement each other; additive threat models allow small, individual changes to every feature of an input while \tmname threat models allow large, uniform changes. Thus, we combine the threat models (see figure \ref{fig:threat_model}) and show that the combination encompasses more potential perturbations than either one separately, as we explain in the following theorem which is stated more precisely in section \ref{combine_tm}.

\begin{theorem}[informal]
Let $\example$ be a grayscale image with $n \geq 2$ pixels. Consider an additive threat model that allows changing each pixel by up to a certain amount, and a \tmname threat model that allows darkening or lightening the entire image by a greater amount. Then the combination of these threat models allows potential perturbations that are not allowed in either constituent threat model.
\end{theorem}

\begin{figure}[t]
    \centering
    \includegraphics[width=0.9\columnwidth]{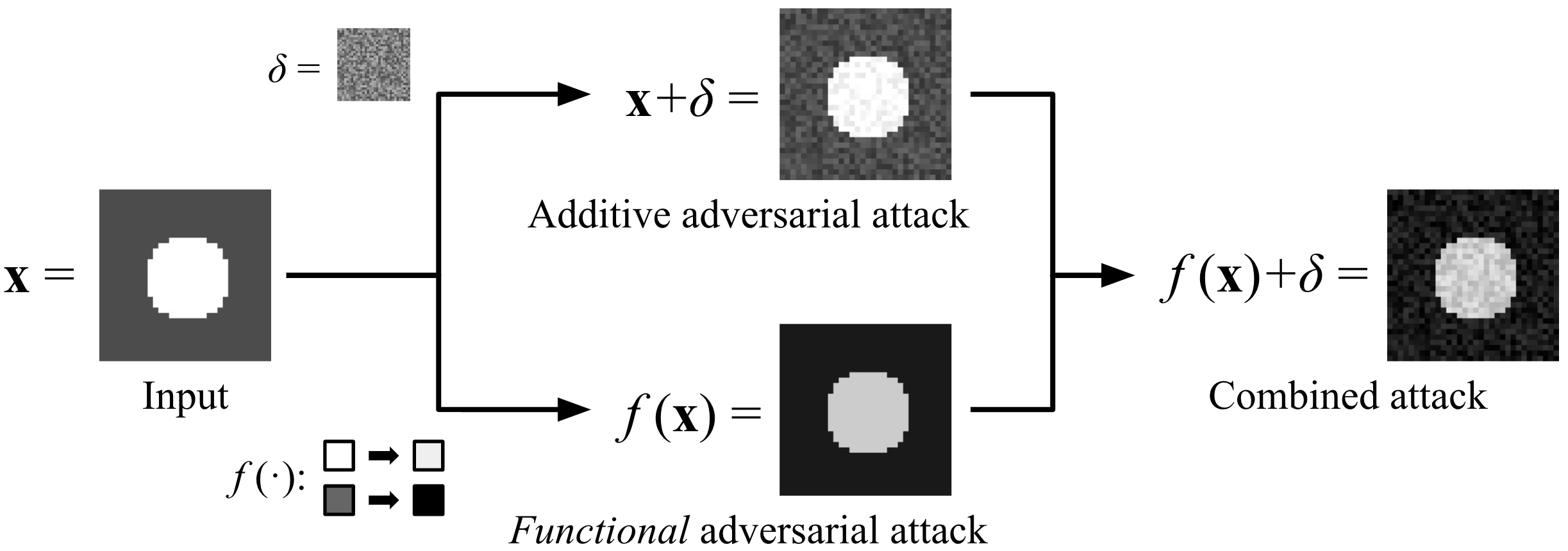}
    \caption{A visualization of an additive adversarial attack, a \tmname adversarial attack, and their combination. The additive attack perturbs each feature (pixel) separately, whereas the \tmname attack applies the same function $\tmfunc(\cdot)$ to every feature.}
    \label{fig:threat_model}
\end{figure}

\Tmname threat models can be used in a variety of domains such as images (e.g. by uniformly changing image colors), speech/audio (e.g. by changing the "accent" of an audio clip), text (e.g. by replacing a word in the entire document with its synonym), or fraud analysis (e.g. by uniformly modifying an actor's financial activities). Moreover, because \tmname perturbations are large and uniform, they may also be easier to use for physical adversarial examples, where the small pixel-level changes created in additive perturbations could be drowned out by environmental noise. 

In this paper, we will focus on one such domain---images---and define \modelname, a \tmname adversarial attack on pixel colors (see figure \ref{fig:examples_imagenet}). In \modelname, we use a flexibly parameterized function $\tmfunc$ to map each pixel color $c$ in the input to a new pixel color $f(c)$ in an adversarial example. We regularize $\tmfunc(\cdot)$ both to ensure that no color is perturbed by more than a certain amount, and to make sure that the mapping is smooth, i.e. similar colors are perturbed in a similar way. We show that \modelname can use colors defined in the standard \ac{rgb} color space and also in CIELUV color space, which results in less perceptually different adversarial examples (see figure \ref{fig:color_space}).

We experiment by attacking defended and undefended classifiers with \modelname, by itself and in combination with other attacks. We find that \modelname is a strong attack, reducing the accuracy of a ResNet-32 trained on CIFAR-10 to 3.0\%. Combinations of \modelname and other attacks are yet more powerful; one such combination lowers a CIFAR-10 classifier's accuracy to 3.6\%, even after adversarial training. This is lower than the previous strongest attack of \citet{jordan_quantifying_2019}. We also demonstrate the fragility of adversarial defenses based on an additive threat model by reducing the accuracy of a classifier trained with TRADES \citep{zhang_theoretically_2019} to 5.7\%. Although one might attempt to mitigate the \modelname attack by converting images to grayscale before classification, which removes color information, we show that this simply decreases a classifier's accuracy (both natural and adversarial). Furthermore, we find that combining \modelname with other attacks improves the strength of the attack without increasing the perceptual difference, as measured by LPIPS \citep{zhang_unreasonable_2018}, of the generated adversarial example.




Our contributions are summarized as follows:

\begin{itemize}[leftmargin=*]
    \item We \textbf{introduce} a novel class of threat models, \tmname adversarial threat models, and combine them with existing threat models. We also describe ways of regularizing \tmname threat models to ensure that generated adversarial examples are imperceptible.
    \item \textbf{Theoretically}, we prove that additive and \tmname threat models combine to create a threat model that encompasses more potential perturbations than either threat model alone. 
    \item \textbf{Experimentally}, we show that \modelname, which uses a \tmname threat model on images, is a strong adversarial attack against image classifiers. To the best of our knowledge, combining \modelname with other attacks leads to the strongest existing attack even after adversarial training.
\end{itemize}

\begin{figure}[t]
    \centering
    \includegraphics[width=0.9\columnwidth]{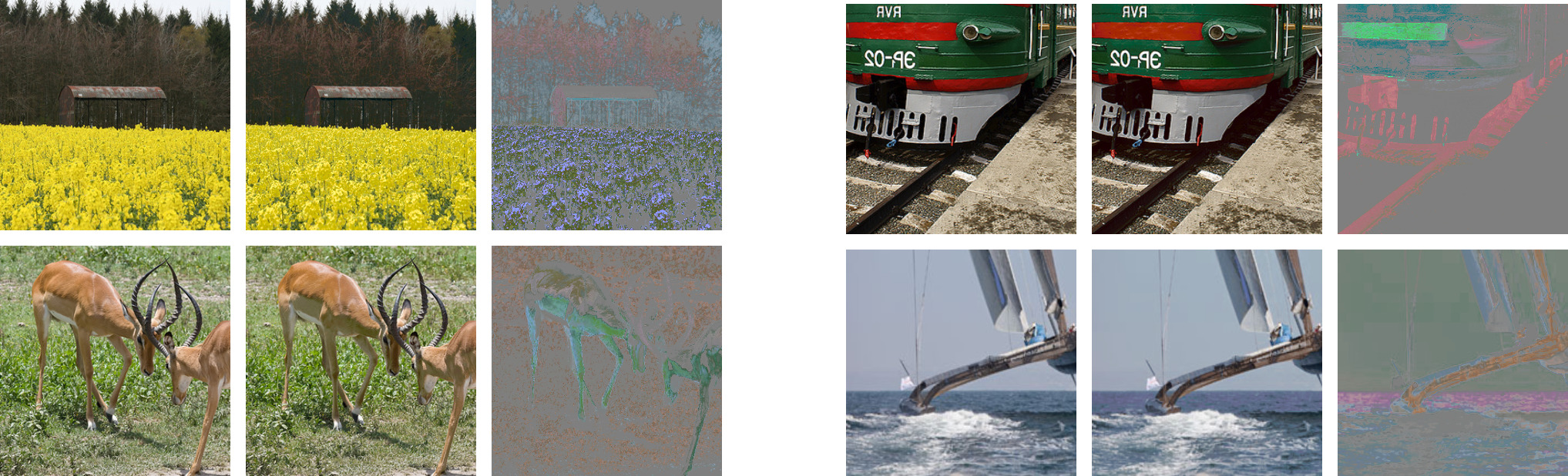}
    \caption{Four ImageNet adversarial examples generated by \modelname against an Inception-v4 classifier. From left to right in each group: original image, adversarial example, magnified difference.
    }
    \label{fig:examples_imagenet}
\end{figure}

\section{Review of Existing Threat Models}

In this section, we define the problem of generating an adversarial example and review existing adversarial threat models and attacks.

\smallparagraph{Problem Definition} Consider a classifier $\clf: \mathcal{X}^n \rightarrow \mathcal{Y}$ from a feature space $\mathcal{X}^n$ to a set of labels $\mathcal{Y}$. Given an input $\example \in \mathcal{X}^n$, an adversarial example is a slight perturbation $\advexample$ of $\example$ such that $\clf(\advexample) \neq \clf(\example)$; that is, $\advexample$ is given a different label than $\example$ by the classifier. Since the aim of an adversarial example is to be perceptually indistinguishable from a normal input, $\advexample$ is usually constrained to be close to $\example$ by some \textit{threat model}. Formally, \citet{jordan_quantifying_2019} define a threat model as a function $t: \mathcal{P}(\mathcal{X}^n) \rightarrow \mathcal{P}(\mathcal{X}^n)$, where $\mathcal{P}$ denotes the power set. The function $t(\cdot)$ maps a set of classifier inputs $\mathcal{S}$ to a set of perturbed inputs $t(\mathcal{S})$ that are imperceptibly different.
With this definition, we can formalize the problem of generating an adversarial example from an input:
\begin{equation*}
    \text{find } \advexample \quad \text{such that } \;
    \clf(\advexample) \neq \clf(\example) \; \text{ and } \;
    \advexample \in t(\{\example\})
\end{equation*}

\smallparagraph{Additive Threat Model} The most common threat model used when generating adversarial examples is the additive threat model. Let $\example = (x_1, \dots, x_n)$, where each $x_i \in \mathcal{X}$ is a feature of $\example$. For instance, $x_i$ could correspond to a pixel in an image or the filterbank energies for a timestep in an audio sample. In an additive threat model, we assume $\advexample = (x_1 + \delta_1, \dots, x_n + \delta_n)$; that is, a value $\delta_i$ is added to each feature of $\example$ to generate the adversarial example $\advexample$. Under this threat model, perceptual similarity is usually enforced by a bound on the norm of $\delta = (\delta_1, \dots, \delta_n)$. Thus, the additive threat model is defined as
\begin{equation*}
    t_\text{add}(\mathcal{S}) \triangleq \left\{(x_1 + \delta_1, \dots, x_n + \delta_n) \mid (x_1, \dots, x_n) \in \mathcal{S}, \|\delta\| \leq \epsilon\right\}.
\end{equation*}
Commonly used norms include $\|\cdot\|_2$ (Euclidean distance), which constrains the sum of squares of the $\delta_i$, $\|\cdot\|_0$, which constrains the number of features can be changed, and $\|\cdot\|_\infty$, which allows changing each feature by up to a certain amount. Note that all of the $\delta_i$ can be modified individually to generate a misclassification, as long as the norm constraint is met. Thus, a small $\epsilon$ is usually necessary because otherwise the input could be made incomprehensible by noise.

Most previous work on generating adversarial examples has employed the additive threat model. This includes gradient-based methods like FGSM \citep{goodfellow_explaining_2015}, DeepFool \citep{moosavi-dezfooli_deepfool:_2016}, and Carlini \& Wagner \citep{carlini_towards_2017}, and gradient-free methods like SPSA \citep{uesato_adversarial_2018} and the Boundary Attack \citep{brendel_decision-based_2018}.

\smallparagraph{Other Threat Models} Some recent work has focused on \textit{spatial threat models}, which allow for slight perturbations of the locations of features in an input rather than perturbations of the features themselves \citep{xiao_spatially_2018,wong_wasserstein_2019,engstrom_rotation_2017}. Others have proposed threat models based on properties of a 3D renderer \citep{zeng_adversarial_2019} and constructing adversarial examples with a GAN \citep{song_constructing_2018}. Finally, some research has focused on coloring-based threat models through modification of an image's hue and saturation \citep{hosseini_semantic_2018}, inverting images \citep{hosseini_limitation_2017}, using a colorization network \citep{bhattad_big_2019}, and applying an affine transformation to colors followed by PGD \citep{zhang_limitations_2019}. See appendix \ref{other_tm} for a discussion of non-additive threat models and comparison to our proposed \tmname threat model.

\section{\Tmname Threat Model}
\label{trans_tm}

In this section, we define \textit{\tmname threat model} and explore its combinations with existing threat models. Recall that in the additive threat model, each feature of an input can only be perturbed by a small amount. Because all the features are changed separately, larger changes could make the input unrecognizable. Our key insight is that larger perturbations to an input should be possible if the dependencies between features are considered.


Unlike the additive threat model, in the \tmname threat model the features $x_i$ are transformed by a single function $\tmfunc: \mathcal{X} \rightarrow \mathcal{X}$, called the perturbation function. That is,
\begin{equation*}
    \advexample = f(\example)= (\tmfunc(x_1), \dots, \tmfunc(x_n))
\end{equation*}
Under this threat model, features which have the same value in the input must be mapped to the same value in the adversarial example.
Even large perturbations allowed by a \tmname threat model may be imperceptible to human eyes because they preserve dependencies between features (for example, shape boundaries and shading in images, see figure \ref{fig:threat_model}). Note that the features $x_i$ which are modified by the perturbation function $\tmfunc(\cdot)$ need not be scalars; depending on the application, vector-valued features may be useful.

\subsection{Regularizing Functional Threat Models}
\label{reg_tm}

In the \tmname threat model, various regularizations can be used to ensure that the change remains imperceptible. In general, we can enforce that $\tmfunc \in \funcfam{}$; $\funcfam{}$ is a family of allowed perturbation functions. For instance, we may want to bound by some small $\epsilon$ the maximum difference between the input and output of the perturbation function. In that case, we will have:
\begin{equation} \label{eq:trans_norm}
    \funcfam{diff} \triangleq \left\{ f: \mathcal{X} \rightarrow \mathcal{X} \mid \forall x_i \in \mathcal{X} \  \|\tmfunc(x_i) - x_i\| \leq \epsilon \right\}
\end{equation}
$\funcfam{diff}$ prevents absolute changes of more than a certain amount. Note that the $\epsilon$ bound may be higher than that of an additive model, since uniform changes are less perceptible. However, this regularization may not be enough to prevent noticeable changes. $\funcfam{diff}$ still includes functions that map similar (but not identical) features very differently. Therefore, a second constraint could be used that forces similar features to be perturbed similarly:
\begin{equation} \label{eq:trans_smooth}
    \funcfam{smooth} \triangleq \left\{ f \mid \forall x_i, x_j \in \mathcal{X} \ \|x_i - x_j\| \leq r \Rightarrow \|(\tmfunc(x_i) - x_i) - (\tmfunc(x_j) - x_j)\| \leq \epsilon_\text{smooth} \right\}
\end{equation}



$\funcfam{smooth}$ requires that similar features are perturbed in the same "direction". For instance, if green pixels in an image are lightened, then yellow-green pixels should be as well.

Depending on the application, these constraints or others may be needed to maintain an imperceptible change. We may want to choose $\funcfam{}$ to be $\funcfam{diff}$, $\funcfam{smooth}$, $\funcfam{diff} \cap \funcfam{smooth}$, or an entirely different family of functions. Once we have chosen an $\funcfam{}$, we can define a corresponding \tmname threat model as
\begin{equation*}
    t_\text{\tm}(\mathcal{S}) \triangleq \left\{ (\tmfunc(x_1), \dots, \tmfunc(x_n)) \mid (x_1, \dots, x_n) \in \mathcal{S}, f \in \funcfam{} \right\}
\end{equation*}


\subsection{Combining Threat Models}
\label{combine_tm}

\citet{jordan_quantifying_2019} argue that combining multiple threat models allows better approximation of the complete set of adversarial perturbations which are imperceptible. Here, we show that combining the additive threat model with a simple \tmname threat model can allow adversarial examples which are not allowable by either model on its own. The following theorem (proved in appendix \ref{theorem_proof}) demonstrates this on images for a combination of an additive threat model which allows changing each pixel by a small, bounded amount and a \tmname threat model which allows darkening or lightening the entire image by up to a larger amount, both of which are arguably imperceptible transformations.

\addtocounter{theorem}{-1}
\begin{theorem}
\label{thm:combined}

Let $\example$ be a grayscale image with $n \geq 2$ pixels, i.e. $\example \in [0, 1]^n = \mathcal{X}^n$. Let $t_\text{add}$ be an additive threat model where the $\ell_\infty$ distance between input and adversarial example is bounded by $\epsilon_1$, i.e. $\|(\delta_1, \dots, \delta_n)\|_\infty \leq \epsilon_1$. Let $t_\text{\tm}$ be a \tmname threat model where $\tmfunc(x) = c \: x$ for some $c \in [1 - \epsilon_2, 1 + \epsilon_2]$ such that $\epsilon_2 > \epsilon_1 > 0$. Let $t_\text{combined} = t_\text{add} \circ t_\text{\tm}$. Then the combined threat model allows adversarial perturbations which are not allowed by either constituent threat model. Formally, if $\mathcal{S} \subseteq \mathcal{X}^n$ contains an image $\example$ that is not dark, that is $\exists\, x_i \text{ s.t. } x_i > \epsilon_1 / \epsilon_2$, then
\begin{equation*}
    t_\text{combined}(\mathcal{S}) \;\supsetneq\;
    t_\text{add}(\mathcal{S}) \cup t_\text{\tm}(\mathcal{S})
    \qquad \textup{or equivalently} \qquad
    \exists \, \advexample \ \text{s.t.} \ 
    \begin{tabular}{l}
        $\advexample \in t_\text{combined}(\mathcal{S})$ \\
        $\advexample \notin t_\text{add}(\mathcal{S}) \cup t_\text{\tm}(\mathcal{S})$
    \end{tabular}
\end{equation*}
\end{theorem}


\section{\modelname: \Tmname Adversarial Attacks on Image Colors}
\label{generation}

\begin{figure}[t]
    \centering
    \includegraphics[width=0.7\columnwidth]{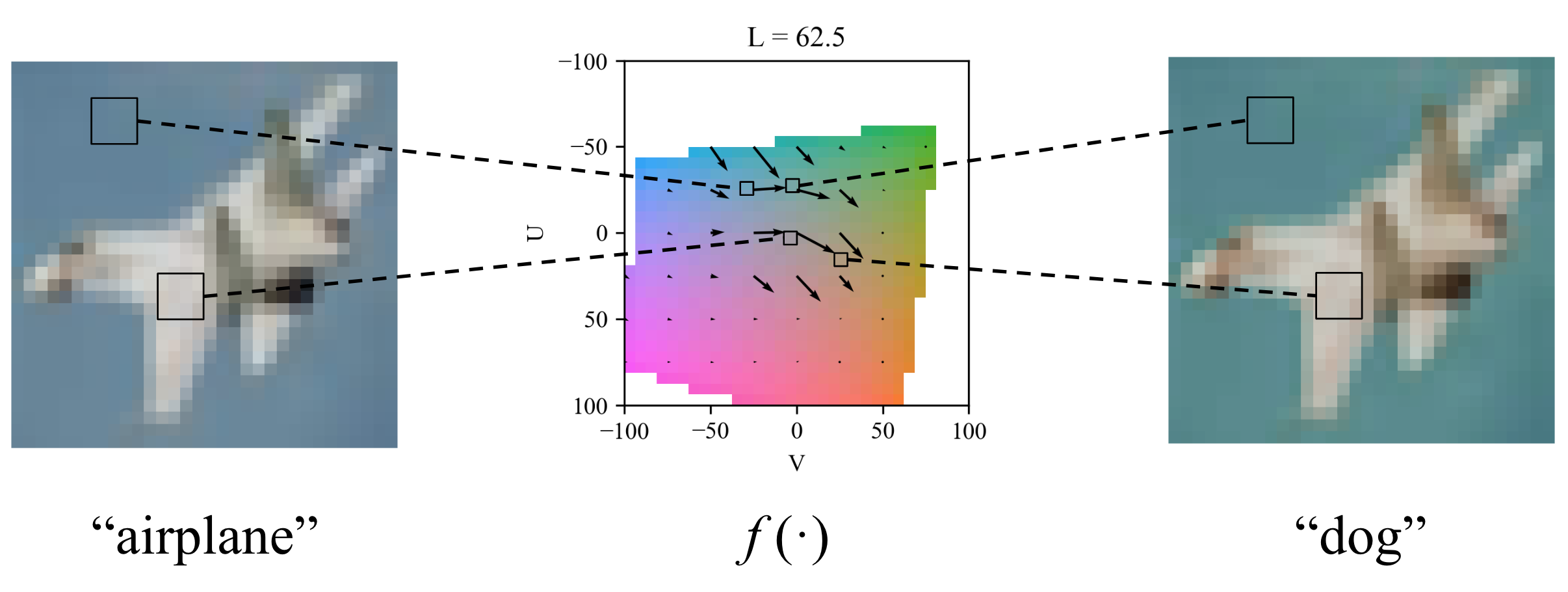}
    \caption{\modelname transforms each pixel in the input image $\example$ (left) by the same function $\tmfunc(\cdot)$ (center) to produce an adversarial example $\advexample$ (right). The perturbation function $\tmfunc$ is shown as a vector field in CIELUV color space.}
    \label{fig:attack}
\end{figure}


In this section, we define \modelname, a novel adversarial attack against image classifiers that leverages a \tmname threat model.
\modelname generates adversarial examples to fool image classifiers by uniformly changing colors of an input image. We treat each pixel $x_i$ in the input image $\example$ as a point in a 3-dimensional color space $\mathcal{C} \subseteq [0, 1]^3$. For instance, $\mathcal{C}$ could be the normal \ac{rgb} color space. In section \ref{color_space}, we discuss our use of alternative color spaces. We leverage a perturbation function $\tmfunc: \mathcal{C} \rightarrow \mathcal{C}$ to produce the adversarial example. Specifically, each pixel in the output $\advexample$ is perturbed from the input $\example$ by applying $\tmfunc(\cdot)$ to the color in that pixel:
\begin{equation*}
    x_i = (c_{i, 1}, c_{i, 2}, c_{i, 3}) \in \mathcal{C} \subseteq [0, 1]^3
    \quad \rightarrow \quad
    \widetilde{x}_i = (\widetilde{c}_{i, 1}, \widetilde{c}_{i, 2}, \widetilde{c}_{i, 3}) = \tmfunc(c_{i, 1}, c_{i, 2}, c_{i, 3})
\end{equation*}
For the purposes of finding an $\tmfunc(\cdot)$ that generates a successful adversarial example, we need a parameterization of the function that allows both flexibility and ease of computation. To accomplish this, we let $\mathcal{G} = {g_1, \dots, g_m} \subseteq [0, 1]^3$ be a discrete grid of points (or point lattice) where $\tmfunc$ is explicitly defined. That is, we define parameters ${\theta_1, \dots, \theta_m}$ and let $f(g_i) = \theta_i$. For points not on the grid, i.e. $x_i \notin \mathcal{G}$, we define $f(x_i)$ using trilinear interpolation. Trilinear interpolation considers the "cube" of the lattice points $g_j$ surrounding the argument $x_i$ and linearly interpolates the explicitly defined $\theta_j$ values at the 8 corners of this cube to calculate $f(x_i)$.


\smallparagraph{Constraints on the perturbation function} We enforce two constraints on $\tmfunc(\cdot)$ to ensure that the crafted adversarial example is indistinguishable from the original image. These constraints are based on slight modifications of $\funcfam{diff}$ and $\funcfam{smooth}$ defined in section \ref{reg_tm}. First, we ensure that no pixel can be perturbed by more than a certain amount along each dimension in color space:
\begin{equation*}
    \funcfam{diff-col} \triangleq
    \left\{
    \tmfunc: \mathcal{C} \rightarrow \mathcal{C}
    \;\middle|\;
    \forall (c_1, c_2, c_3) \in \mathcal{G} \quad
    |c_i - \widetilde{c}_i| < \epsilon_i \quad
    i = 1, 2, 3
    \right\}
\end{equation*}
This particular formulation allows us to set different bounds $(\epsilon_1, \epsilon_2, \epsilon_3)$ on the maximum perturbation along each dimension in color space. We also define a constraint based on $\funcfam{smooth}$, but instead of using a radius parameter $r$ as in (\ref{eq:trans_smooth}) we consider the neighbors $\mathcal{N}(g_j)$ of each lattice point $g_j$ in the grid $\mathcal{G}$:
\begin{equation*}
    \funcfam{smooth-col} \triangleq
    \left\{
    f: \mathcal{C} \rightarrow \mathcal{C} \mid
    \forall g_j \in \mathcal{X}, g_k \in \mathcal{N}(g_j) \quad
    \|(\tmfunc(g_j) - g_j) - (\tmfunc(g_k) - g_k)\|_2 \leq \epsilon_\text{smooth} \right\}
\end{equation*}
In the above, $\|\cdot\|_2$ is the $\ell_2$ (Euclidean) norm in the color space $\mathcal{C}$. We define our set of allowable perturbation functions as $\funcfam{col} = \funcfam{diff-col} \cap \funcfam{smooth-col}$ with parameters $(\epsilon_1, \epsilon_2, \epsilon_3, \epsilon_\text{smooth})$.

\smallparagraph{Optimization} To generate an adversarial example with \modelname, we wish to minimize $\mathcal{L}_\text{adv}(f, x)$ subject to $f \in \funcfam{col}$, where $\mathcal{L}_\text{adv}$ enforces the goal of generating an adversarial example that is misclassified and is defined as the $f_6$ loss from \citet{carlini_towards_2017}, where $\clf(\example)_i$ represents the classifier's $i$th logit:

\begin{equation}
\label{cw_f6}
    \mathcal{L}_\text{adv}(f, x) = \max \left(\max_{i \neq y} (\clf(\advexample)_i - \clf(\advexample)_y), 0\right)
\end{equation}

When solving this constrained minimization problem, it is easy to constrain $f \in \funcfam{diff-col}$ by clipping the perturbation of each color to be within the $\epsilon_i$ bounds. However, it is difficult to enforce $f \in \funcfam{smooth-col}$ directly. Thus, we instead solve a Lagrangian relaxation
where the smoothness constraint is replaced by an additional regularization term:
\begin{equation}
    \label{eq:optim_relax} \tag{$*$}
    \argmin_{f \in \funcfam{diff-col}} \;
    \mathcal{L}_\text{adv}(\tmfunc, \example) + \lambda \, 
    \mathcal{L}_\text{smooth}(\tmfunc)
\end{equation}
\begin{equation*}
    \mathcal{L}_\text{smooth}(f) \triangleq \sum_{g_j \in \mathcal{G}} \sum_{g_k \in \mathcal{N}(g_j)} \|(\tmfunc(g_j) - g_j) - (\tmfunc(g_k) - g_k)\|_2
\end{equation*}
Our $\mathcal{L}_\text{smooth}$ is similar to the loss function used by \citet{xiao_spatially_2018} to ensure a smooth flow field. We use the \ac{pgd} optimization algorithm to solve (\ref{eq:optim_relax}).

\subsection{RGB vs. LUV Color Space}
\label{color_space}

\begin{figure}[t]
    \centering
    \includegraphics[width=1.0\columnwidth]{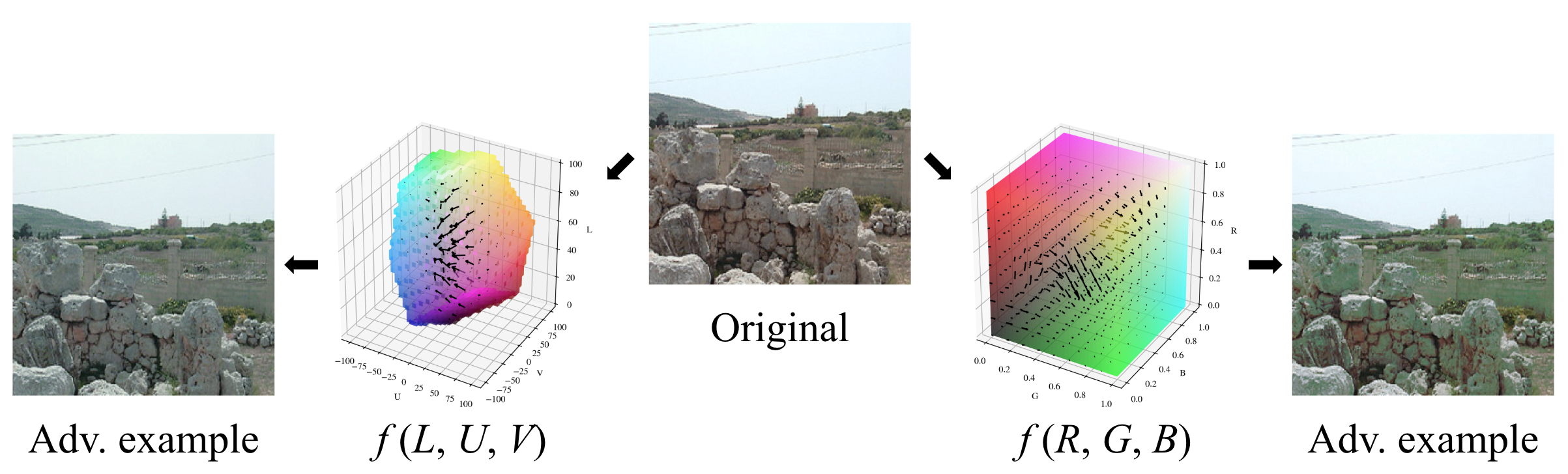}
    \caption{The color space used affects the adversarial example produced by \modelname. The original image at center is attacked by \modelname with the CIELUV color space (left) and RGB color space (right). The RGB color space results in noticeable bright green artifacts in the adversarial example, while the perceptually accurate CIELUV color space produces a more realistic perturbation.}
    \label{fig:color_space}
\end{figure}

Most image classifiers take as input an array of pixels specified in \ac{rgb} color space, but the \ac{rgb} color space has two disadvantages.
The $\ell_p$ distance between points in \ac{rgb} color space is weakly correlated with the perceptual difference between the colors they represent. Also, \ac{rgb} gives no separation between the luma (brightness) and chroma (hue/colorfulness) of a color. 

In contrast, the CIELUV color space separates luma from chroma and places colors such that the Euclidean distance between them is roughly equivalent to the perceptual difference \citep{schwiegerling_field_2004}.
CIELUV presents a color by three components $(L, U, V)$; $L$ is the luma while $U$ and $V$ together define the chroma.
We run experiments using both \ac{rgb} and CIELUV color spaces. CIELUV allows us to regularize the perturbation function $f(\cdot)$ perceptually accurately (see figure \ref{fig:color_space} and appendix \ref{regularization}).
We experimented with the \ac{hsv} and YPbPr color spaces as well; however, neither is perceptually accurate and the \ac{hsv} transformation from RGB is difficult to differentiate (see appendix \ref{other_color_space}).

\section{Experiments}

We evaluate \modelname against defended and undefended neural networks on CIFAR-10 \citep{krizhevsky_learning_2009} and ImageNet \citep{russakovsky_imagenet_2015}. For CIFAR-10 we evaluate the attack against a ResNet-32 \citep{he_deep_2016} and for ImageNet we evaluate against an Inception-v4 network \citep{szegedy_inception-v4_2017}. We also consider all combinations of \modelname with \textit{delta} attacks, which use an $\ell_\infty$ additive threat model with bound $\epsilon = 8/255$, and the \textit{StAdv} attack of \citet{xiao_spatially_2018} that perturbs images spatially through a flow field. See appendix \ref{parameters} for a full discussion of the hyperparameters and computing infrastructure used in our experiments. We release our code at \url{https://github.com/cassidylaidlaw/ReColorAdv}.

\subsection{Adversarial Training}

We first experiment by attacking adversarially trained models with \modelname and other attacks. For each combination of attacks, we adversarially train a ResNet-32 on CIFAR-10 against that particular combination. We attack each of these adversarially trained models with all combinations of attacks. The results of this experiment are shown in the first part of table \ref{tab:accuracy}.

\smallparagraph{Combination attacks are most powerful}
As expected, combinations of attacks are the strongest against defended and undefended classifiers. In particular, the \modelname + StAdv + delta attack often resulted in the lowest classifier accuracy. The accuracy of only 3.6\% after adversarially training against the \modelname + StAdv + delta attack is the lowest we know of.

\smallparagraph{Transferability of robustness across perturbation types}
While adversarial attack "transferability" often refers to the ability of attacks to transfer between models \citep{papernot_transferability_2016}, here we investigate to what degree a model robust to one type of adversarial perturbation is robust to other types of perturbations, similarly to \citet{kang_transfer_2019}. To some degree, the perturbations investigated are orthogonal; that is, a model trained against a particular type of perturbation is less effective against others. StAdv is especially separate from the other two attacks; models trained against StAdv attacks are still very vulnerable to \modelname and delta attacks. However, the \modelname and delta attacks allow more transferable robustness between each other. These results are likely due to the fact that both the delta and \modelname attacks operate on a per-pixel basis, whereas the StAdv attack allows spatial movement of features across pixels.

\smallparagraph{Effect of color space}
The \modelname attack using CIELUV color space is stronger than that using \ac{rgb} color space. In addition, the CIELUV color space produces less perceptible perturbations (see figure \ref{fig:color_space}). This highlights the need for using perceptually accurate models of color when designing and defending against adversarial examples.

\begin{table}[t]
    \centering
    \caption{Accuracy of adversarially trained models against various combinations of attacks on CIFAR-10. Columns correspond to attacks and rows correspond to models trained against a particular attack. C(-RGB) is \modelname using CIELUV (\ac{rgb}) color space, D is delta attack, and S is StAdv attack. TRADES is the method of \citet{zhang_theoretically_2019}. For classifiers marked (B\&W), the images are converted to black-and-white before classification.}
    \begin{tabular}{l|rrrrrrrrr}
        \toprule
        & \multicolumn{9}{c}{\bf Attack} \\
        \bf Defense $\downarrow$ & \bf None & \bf C-RGB & \bf C & \bf D & \bf S & \bf C+S & \bf C+D & \bf S+D & \bf C+S+D \\
        \midrule
        \bf Undefended & 92.2 & 5.9 & 3.0 & \textbf{0.0} & 0.9 & 0.8 & \textbf{0.0} & \textbf{0.0} & \textbf{0.0} \\
        \bf C & 88.7 & 43.5 & 45.8 & 5.7 & 3.6 & 3.4 & 0.9 & \textbf{0.2} & \textbf{0.2} \\
        \bf D & 84.8 & 74.9 & 50.6 & 30.6 & 16.0 & 11.7 & 8.9 & 2.7 & \textbf{2.2} \\
        \bf S & 82.7 & 16.9 & 8.0 & 0.5 & 26.2 & 4.8 & \textbf{0.0} & 0.1 & \textbf{0.0} \\
        \bf C+S & 89.5 & 31.7 & 23.0 & 0.7 & 10.9 & 8.7 & 0.5 & 0.6 & \textbf{0.4} \\
        \bf C+D & 88.5 & 36.3 & 19.5 & 7.5 & \textbf{2.7} & 2.8 & 5.2 & 4.1 & 4.6 \\
        \bf S+D & 82.1 & 66.9 & 42.7 & 35.4 & 21.9 & 13.4 & 12.2 & 7.6 & \textbf{4.1} \\
        \bf C+S+D & 88.9 & 30.6 & 17.2 & 7.3 & 3.5 & \textbf{3.3} & 5.5 & 3.7 & 3.6 \\
        \midrule
        \bf TRADES & 84.4 & 81.3 & 59.2 & 53.6 & 26.6 & 17.5 & 22.0 & 8.6 & \textbf{5.7} \\
        \midrule
        \bf Undefended (B\&W) & 88.3 & 5.3 & 4.1 & \textbf{0.0} & 0.9 & 0.6 & \textbf{0.0} & \textbf{0.0} & \textbf{0.0} \\
        \bf C (B\&W) & 85.8 & 40.8 & 38.9 & 4.2 & 2.5 & 2.5 & 0.5 & \textbf{0.1} & 0.2 \\
        \bottomrule
    \end{tabular}
    \label{tab:accuracy}
\end{table}

\begin{figure}[t]
    \centering
    \begin{minipage}[b]{0.48\columnwidth}
        \begin{tabularx}{\columnwidth}{YYYYY}
            Orig. & C & D & C+D & C+S+D
        \end{tabularx}
        \includegraphics[width=\columnwidth]{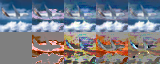}
    \end{minipage}
    \quad
    \begin{minipage}[b]{0.48\columnwidth}
        \begin{tabularx}{\columnwidth}{YYYYY}
            Orig. & C & D & C+D & C+S+D
        \end{tabularx}
        \includegraphics[width=\columnwidth]{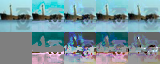}
    \end{minipage}
    \caption{Adversarial examples generated with combinations of attacks against a CIFAR-10 WideResNet \citep{zagoruyko_wide_2016} trained using TRADES; the difference from the original is shown below each example. Combinations of attacks tend to produce less perceptible changes than the attacks do separately.}
    \label{fig:examples_cifar}
\end{figure}

\subsection{Other Defenses}

\smallparagraph{TRADES}
TRADES is a training algorithm for deep neural networks that aims to improve robustness to adversarial examples by optimizing a surrogate loss \citep{zhang_theoretically_2019}. The algorithm is designed around an additive threat model, but we evaluate a TRADES-trained classifier on all combinations of attacks (see the second part of table \ref{tab:accuracy}). This is the best defense method against almost all attacks, despite having been trained based on just an additive threat model. However, the combined \modelname + StAdv + delta attack still reduces the accuracy of the classifier to just 5.7\%.

\smallparagraph{Grayscale conversion}
Since \modelname attacks input images by changing their colors, one might attempt to mitigate the attack by converting all images to grayscale before classification. This could reduce the potential perturbations available to \modelname since altering the chroma of a color would not affect the grayscale image; only changes in luma would. We train models on CIFAR-10 that convert all images to grayscale as a preprocessing step both with and without adversarial training against \modelname. The results of this experiment (see the third part of table \ref{tab:accuracy}) show that conversion to grayscale is not a viable defense against \modelname. In fact, the natural accuracy and robustness against almost all attacks decreases when applying grayscale conversion.

\subsection{Perceptual Distance}
\label{lpips}

We quantify the perceptual distortion caused by \modelname attacks using the \ac{lpips} metric, a distance measure between images based on deep network activations which has been shown to correlate with human perception \citep{zhang_unreasonable_2018}. We combine \modelname and delta attacks and vary the bound of each attack (see figure \ref{fig:lpips}). We find that the attacks can be combined without much increase, or with even sometimes a \textit{decrease}, in perceptual difference. As \citet{jordan_quantifying_2019} find for combinations of StAdv and delta attacks, the lowest perceptual difference at a particular attack strength is achieved by a combination of \modelname and delta attacks.

\begin{figure}[t]
    \centering
    \includegraphics[width=0.9\columnwidth]{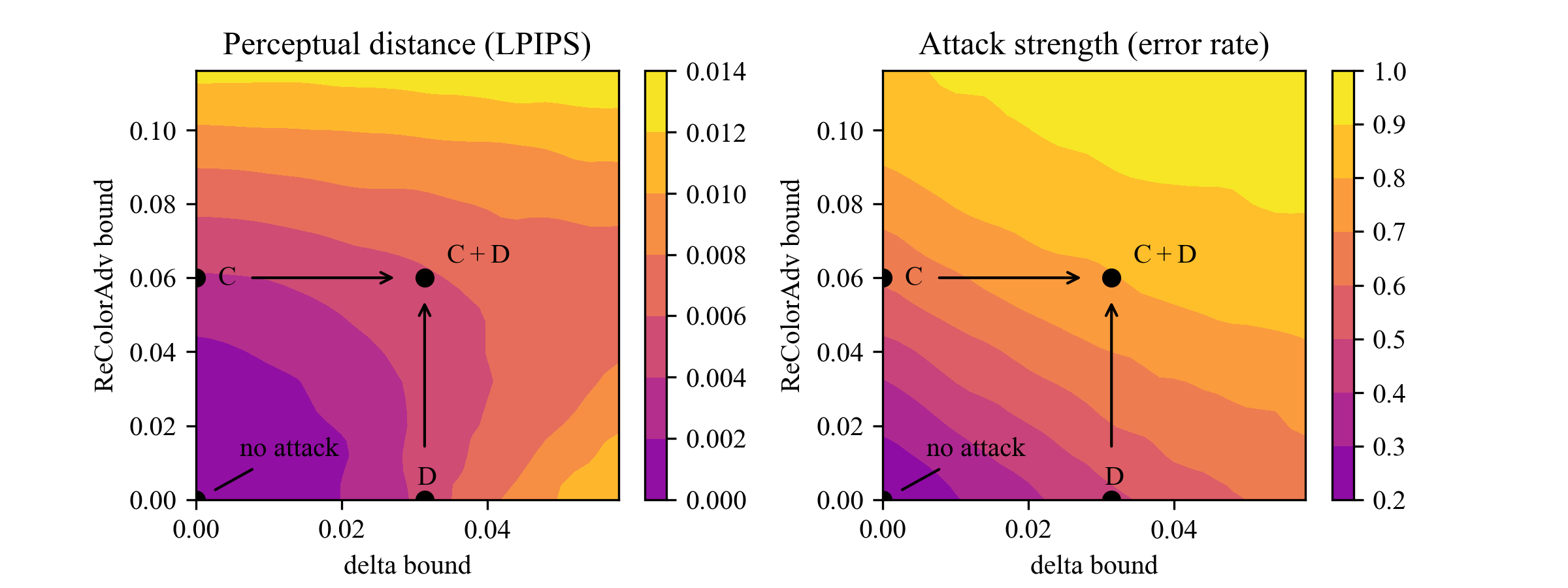}
    \caption{The perceptual distortion (LPIPS) and strength (error rate) of combinations of \modelname and delta attacks with various bounds. The annotated points mark the bounds used in other experiments: C is \modelname, D is a delta attack, and C+D is their combination. Combining the attacks does not increase perceptable change by much (left), but it greatly increases attack strength (right).}
    \label{fig:lpips}
\end{figure}

\section{Conclusion}

We have presented \textit{\tmname threat models} for adversarial examples, which allow large, uniform changes to an input. They can be combined with additive threat models to provably increase the potential perturbations allowed in an adversarial attack. In practice, the \modelname attack, which leverages a \tmname threat model against image pixel colors, is a strong adversarial attack on image classifiers. It can also be combined with other attacks to produce yet more powerful attacks---even after adversarial training---without a significant increase in perceptual distortion. Besides images, \tmname adversarial attacks could be designed for audio, text, and other domains. It will be crucial to develop defense methods against these attacks, which encompass a more complete threat model of which potential adversarial examples are imperceptible to humans.

\section*{Acknowledgements}

This work was supported in part by NSF award CDS\&E:1854532 and award HR00111990077.

\bibliographystyle{plainnat}
\bibliography{paper}

\begin{thebibliography}{32}
\providecommand{\natexlab}[1]{#1}
\providecommand{\url}[1]{\texttt{#1}}
\expandafter\ifx\csname urlstyle\endcsname\relax
  \providecommand{\doi}[1]{doi: #1}\else
  \providecommand{\doi}{doi: \begingroup \urlstyle{rm}\Url}\fi

\bibitem[Bhattad et~al.(2019)Bhattad, Chong, Liang, Li, and
  Forsyth]{bhattad_big_2019}
Anand Bhattad, Min~Jin Chong, Kaizhao Liang, Bo~Li, and David~A. Forsyth.
\newblock Big but {Imperceptible} {Adversarial} {Perturbations} via {Semantic}
  {Manipulation}.
\newblock \emph{arXiv preprint arXiv:1904.06347}, 2019.

\bibitem[Brendel et~al.(2018)Brendel, Rauber, and
  Bethge]{brendel_decision-based_2018}
Wieland Brendel, Jonas Rauber, and Matthias Bethge.
\newblock Decision-{Based} {Adversarial} {Attacks}: {Reliable} {Attacks}
  {Against} {Black}-{Box} {Machine} {Learning} {Models}.
\newblock \emph{International Conference on Learning Representations}, February
  2018.

\bibitem[Carlini and Wagner(2017)]{carlini_towards_2017}
Nicholas Carlini and David Wagner.
\newblock Towards {Evaluating} the {Robustness} of {Neural} {Networks}.
\newblock In \emph{2017 {IEEE} {Symposium} on {Security} and {Privacy} ({SP})},
  pages 39--57. IEEE, 2017.

\bibitem[Engstrom et~al.(2017)Engstrom, Tran, Tsipras, Schmidt, and
  Madry]{engstrom_rotation_2017}
Logan Engstrom, Brandon Tran, Dimitris Tsipras, Ludwig Schmidt, and Aleksander
  Madry.
\newblock A {Rotation} and a {Translation} {Suffice}: {Fooling} {CNNs} with
  {Simple} {Transformations}.
\newblock \emph{arXiv preprint arXiv:1712.02779}, 2017.

\bibitem[Goodfellow et~al.(2015)Goodfellow, Shlens, and
  Szegedy]{goodfellow_explaining_2015}
Ian Goodfellow, Jonathon Shlens, and Christian Szegedy.
\newblock Explaining and {Harnessing} {Adversarial} {Examples}.
\newblock In \emph{International {Conference} on {Learning} {Representations}},
  2015.

\bibitem[He et~al.(2016)He, Zhang, Ren, and Sun]{he_deep_2016}
Kaiming He, Xiangyu Zhang, Shaoqing Ren, and Jian Sun.
\newblock Deep {Residual} {Learning} for {Image} {Recognition}.
\newblock In \emph{Proceedings of the {IEEE} {Conference} on {Computer}
  {Vision} and {Pattern} {Recognition}}, pages 770--778, 2016.

\bibitem[Hosseini and Poovendran(2018)]{hosseini_semantic_2018}
Hossein Hosseini and Radha Poovendran.
\newblock Semantic {Adversarial} {Examples}.
\newblock In \emph{Proceedings of the {IEEE} {Conference} on {Computer}
  {Vision} and {Pattern} {Recognition} {Workshops}}, pages 1614--1619, 2018.

\bibitem[Hosseini et~al.(2017)Hosseini, Xiao, Jaiswal, and
  Poovendran]{hosseini_limitation_2017}
Hossein Hosseini, Baicen Xiao, Mayoore Jaiswal, and Radha Poovendran.
\newblock On the {Limitation} of {Convolutional} {Neural} {Networks} in
  {Recognizing} {Negative} {Images}.
\newblock In \emph{16th {IEEE} {International} {Conference} on {Machine}
  {Learning} and {Applications} ({ICMLA})}, pages 352--358. IEEE, 2017.

\bibitem[Idelbayev(2018)]{idelbayev_proper_2018}
Yerlan Idelbayev.
\newblock Proper {ResNet} {Implementation} for {CIFAR}10/{CIFAR}100 in pytorch,
  2018.
\newblock URL \url{https://github.com/akamaster/pytorch_resnet_cifar10}.
\newblock original-date: 2018-01-15T09:50:56Z.

\bibitem[Jordan et~al.(2019)Jordan, Manoj, Goel, and
  Dimakis]{jordan_quantifying_2019}
Matt Jordan, Naren Manoj, Surbhi Goel, and Alexandros~G. Dimakis.
\newblock Quantifying {Perceptual} {Distortion} of {Adversarial} {Examples}.
\newblock \emph{arXiv preprint arXiv:1902.08265}, February 2019.
\newblock arXiv: 1902.08265.

\bibitem[Kang et~al.(2019)Kang, Sun, Brown, Hendrycks, and
  Steinhardt]{kang_transfer_2019}
Daniel Kang, Yi~Sun, Tom Brown, Dan Hendrycks, and Jacob Steinhardt.
\newblock Transfer of {Adversarial} {Robustness} {Between} {Perturbation}
  {Types}.
\newblock \emph{arXiv preprint arXiv:1905.01034}, 2019.

\bibitem[Kingma and Ba(2014)]{kingma_adam:_2014}
Diederik~P. Kingma and Jimmy Ba.
\newblock Adam: {A} {Method} for {Stochastic} {Optimization}.
\newblock \emph{arXiv preprint arXiv:1412.6980}, 2014.

\bibitem[Krizhevsky and Hinton(2009)]{krizhevsky_learning_2009}
Alex Krizhevsky and Geoffrey Hinton.
\newblock Learning {Multiple} {Layers} of {Features} from {Tiny} {Images}.
\newblock Technical report, Citeseer, 2009.

\bibitem[Madry et~al.(2018)Madry, Makelov, Schmidt, Tsipras, and
  Vladu]{madry_towards_2018}
Aleksander Madry, Aleksandar Makelov, Ludwig Schmidt, Dimitris Tsipras, and
  Adrian Vladu.
\newblock Towards {Deep} {Learning} {Models} {Resistant} to {Adversarial}
  {Attacks}.
\newblock 2018.

\bibitem[Moosavi-Dezfooli et~al.(2016)Moosavi-Dezfooli, Fawzi, and
  Frossard]{moosavi-dezfooli_deepfool:_2016}
Seyed-Mohsen Moosavi-Dezfooli, Alhussein Fawzi, and Pascal Frossard.
\newblock Deepfool: a {Simple} and {Accurate} {Method} to {Fool} {Deep}
  {Neural} {Networks}.
\newblock In \emph{Proceedings of the {IEEE} {Conference} on {Computer}
  {Vision} and {Pattern} {Recognition}}, pages 2574--2582, 2016.

\bibitem[Papernot et~al.(2016{\natexlab{a}})Papernot, McDaniel, and
  Goodfellow]{papernot_transferability_2016}
Nicolas Papernot, Patrick McDaniel, and Ian Goodfellow.
\newblock Transferability in {Machine} {Learning}: from {Phenomena} to
  {Black}-{Box} {Attacks} using {Adversarial} {Samples}.
\newblock \emph{arXiv preprint arXiv:1605.07277}, 2016{\natexlab{a}}.

\bibitem[Papernot et~al.(2016{\natexlab{b}})Papernot, McDaniel, Jha,
  Fredrikson, Celik, and Swami]{papernot_limitations_2016}
Nicolas Papernot, Patrick McDaniel, Somesh Jha, Matt Fredrikson, Z~Berkay
  Celik, and Ananthram Swami.
\newblock The {Limitations} of {Deep} {Learning} in {Adversarial} {Settings}.
\newblock In \emph{{IEEE} {European} {Symposium} on {Security} and {Privacy}
  ({EuroS}\&{P})}, pages 372--387. IEEE, 2016{\natexlab{b}}.

\bibitem[Papernot et~al.(2016{\natexlab{c}})Papernot, McDaniel, Wu, Jha, and
  Swami]{papernot_distillation_2016}
Nicolas Papernot, Patrick McDaniel, Xi~Wu, Somesh Jha, and Ananthram Swami.
\newblock Distillation as a {Defense} to {Adversarial} {Perturbations}
  {Against} {Deep} {Neural} {Networks}.
\newblock In \emph{{IEEE} {Symposium} on {Security} and {Privacy} ({SP})},
  pages 582--597. IEEE, 2016{\natexlab{c}}.

\bibitem[Paszke et~al.(2017)Paszke, Gross, Chintala, Chanan, Yang, DeVito, Lin,
  Desmaison, Antiga, and Lerer]{paszke_automatic_2017}
Adam Paszke, Sam Gross, Soumith Chintala, Gregory Chanan, Edward Yang, Zachary
  DeVito, Zeming Lin, Alban Desmaison, Luca Antiga, and Adam Lerer.
\newblock Automatic {Differentiation} in {PyTorch}.
\newblock In \emph{{NIPS}-{W}}, 2017.

\bibitem[Russakovsky et~al.(2015)Russakovsky, Deng, Su, Krause, Satheesh, Ma,
  Huang, Karpathy, Khosla, and Bernstein]{russakovsky_imagenet_2015}
Olga Russakovsky, Jia Deng, Hao Su, Jonathan Krause, Sanjeev Satheesh, Sean Ma,
  Zhiheng Huang, Andrej Karpathy, Aditya Khosla, and Michael Bernstein.
\newblock {ImageNet} {Large} {Scale} {Visual} {Recognition} {Challenge}.
\newblock \emph{International Journal of Computer Vision}, 115\penalty0
  (3):\penalty0 211--252, 2015.

\bibitem[Schwiegerling(2004)]{schwiegerling_field_2004}
Jim Schwiegerling.
\newblock \emph{Field {Guide} to {Visual} and {Ophthalmic} {Optics}}.
\newblock SPIE Publications, Bellingham, Wash, November 2004.
\newblock ISBN 978-0-8194-5629-8.

\bibitem[Song et~al.(2018)Song, Shu, Kushman, and
  Ermon]{song_constructing_2018}
Yang Song, Rui Shu, Nate Kushman, and Stefano Ermon.
\newblock Constructing {Unrestricted} {Adversarial} {Examples} with
  {Generative} {Models}.
\newblock In \emph{Proceedings of the 32nd {International} {Conference} on
  {Neural} {Information} {Processing} {Systems}}, pages 8322--8333. Curran
  Associates Inc., 2018.

\bibitem[Szegedy et~al.(2014)Szegedy, Zaremba, Sutskever, Bruna, Erhan,
  Goodfellow, and Fergus]{szegedy_intriguing_2014}
Christian Szegedy, Wojciech Zaremba, Ilya Sutskever, Joan Bruna, Dumitru Erhan,
  Ian Goodfellow, and Rob Fergus.
\newblock Intriguing {Properties} of {Neural} {Networks}.
\newblock In \emph{International {Conference} on {Learning} {Representations}},
  2014.

\bibitem[Szegedy et~al.(2017)Szegedy, Ioffe, Vanhoucke, and
  Alemi]{szegedy_inception-v4_2017}
Christian Szegedy, Sergey Ioffe, Vincent Vanhoucke, and Alexander~A. Alemi.
\newblock Inception-v4, {Inception}-{ResNet} and the {Impact} of {Residual}
  {Connections} on {Learning}.
\newblock In \emph{Thirty-{First} {AAAI} {Conference} on {Artificial}
  {Intelligence}}, 2017.

\bibitem[Uesato et~al.(2018)Uesato, O’Donoghue, Kohli, and
  Oord]{uesato_adversarial_2018}
Jonathan Uesato, Brendan O’Donoghue, Pushmeet Kohli, and Aaron Oord.
\newblock Adversarial {Risk} and the {Dangers} of {Evaluating} {Against} {Weak}
  {Attacks}.
\newblock In \emph{International {Conference} on {Machine} {Learning}}, pages
  5025--5034, July 2018.
\newblock URL \url{http://proceedings.mlr.press/v80/uesato18a.html}.

\bibitem[Wong et~al.(2019)Wong, Schmidt, and Kolter]{wong_wasserstein_2019}
Eric Wong, Frank~R. Schmidt, and J.~Zico Kolter.
\newblock Wasserstein {Adversarial} {Examples} via {Projected} {Sinkhorn}
  {Iterations}.
\newblock \emph{arXiv preprint arXiv:1902.07906}, February 2019.
\newblock arXiv: 1902.07906.

\bibitem[Xiao et~al.(2018)Xiao, Zhu, Li, He, Liu, and
  Song]{xiao_spatially_2018}
Chaowei Xiao, Jun-Yan Zhu, Bo~Li, Warren He, Mingyan Liu, and Dawn Song.
\newblock Spatially {Transformed} {Adversarial} {Examples}.
\newblock \emph{arXiv preprint arXiv:1801.02612}, 2018.

\bibitem[Zagoruyko and Komodakis(2016)]{zagoruyko_wide_2016}
Sergey Zagoruyko and Nikos Komodakis.
\newblock Wide {Residual} {Networks}.
\newblock In \emph{British {Machine} {Vision} {Conference}}. British Machine
  Vision Association, 2016.

\bibitem[Zeng et~al.(2019)Zeng, Liu, Wang, Qiu, Xie, Tai, Tang, and
  Yuille]{zeng_adversarial_2019}
Xiaohui Zeng, Chenxi Liu, Yu-Siang Wang, Weichao Qiu, Lingxi Xie, Yu-Wing Tai,
  Chi~Keung Tang, and Alan~L. Yuille.
\newblock Adversarial {Attacks} {Beyond} the {Image} {Space}.
\newblock In \emph{Proceedings of the {IEEE} {Conference} on {Computer}
  {Vision} and {Pattern} {Recognition}}, 2019.
\newblock arXiv: 1711.07183.

\bibitem[Zhang et~al.(2019{\natexlab{a}})Zhang, Yu, Jiao, Xing, Ghaoui, and
  Jordan]{zhang_theoretically_2019}
Hongyang Zhang, Yaodong Yu, Jiantao Jiao, Eric~P. Xing, Laurent~El Ghaoui, and
  Michael~I. Jordan.
\newblock Theoretically {Principled} {Trade}-off {Between} {Robustness} and
  {Accuracy}.
\newblock In \emph{{ICML} 2019}, 2019{\natexlab{a}}.

\bibitem[Zhang et~al.(2019{\natexlab{b}})Zhang, Chen, Song, Boning, Dhillon,
  and Hsieh]{zhang_limitations_2019}
Huan Zhang, Hongge Chen, Zhao Song, Duane~S. Boning, Inderjit~S. Dhillon, and
  Cho-Jui Hsieh.
\newblock The {Limitations} of {Adversarial} {Training} and the {Blind}-{Spot}
  {Attack}.
\newblock \emph{International Conference on Learning Representations},
  2019{\natexlab{b}}.

\bibitem[Zhang et~al.(2018)Zhang, Isola, Efros, Shechtman, and
  Wang]{zhang_unreasonable_2018}
Richard Zhang, Phillip Isola, Alexei~A. Efros, Eli Shechtman, and Oliver Wang.
\newblock The {Unreasonable} {Effectiveness} of {Deep} {Features} as a
  {Perceptual} {Metric}.
\newblock In \emph{Proceedings of the {IEEE} {Conference} on {Computer}
  {Vision} and {Pattern} {Recognition}}, pages 586--595, 2018.

\end{thebibliography}

\pagebreak
\appendix

\section{Combining the Additive and \Tmname Threat Models}

Here we provide a proof of Theorem \ref{thm:combined}.

\paragraph{Threat model} 

Let $\example$ be a grayscale image with $n \geq 2$ pixels, i.e. $\example \in [0, 1]^n = \mathcal{X}^n$. Let $t_\text{add}$ be an additive threat model where the $\ell_\infty$ distance between input and adversarial example is bounded by $\epsilon_1$, i.e. $\|(\delta_1, \dots, \delta_n)\|_\infty \leq \epsilon_1$. Let $t_\text{\tm}$ be a \tmname threat model where $\tmfunc(x) = c \: x$ for some $c \in [1 - \epsilon_2, 1 + \epsilon_2]$ and let $\epsilon_2 > \epsilon_1 > 0$. The additive threat model allows individually changing each pixel's value by up to $\epsilon_1$; the \tmname threat model allows darkening or lightening the entire image by up to a proportion of $\epsilon_2$. Both of these are arguably imperceptible perturbations for small enough $\epsilon_1$ and $\epsilon_2$. We also consider $t_\text{combined} = t_\text{add} \circ t_\text{\tm}$:

\begin{equation}
\begin{aligned}
    \label{eq:combined_tm}
    t_\text{combined}(\mathcal{S}) \triangleq \left\{(c \: x_1 + \delta_1, \dots, c \: x_n + \delta_n)
    \;\middle|\;
    \begin{tabular}{@{}l@{}}
        $(x_1, \dots, x_n) \in \mathcal{S}$ \\
        $|\delta_i| \leq \epsilon_1$ \\
        $c \in [1 - \epsilon_2, 1 + \epsilon_2]$
     \end{tabular}\right\}
\end{aligned}
\end{equation}

This combined threat model allows darkening or lightening the image, followed by changing each pixel value individually by a small amount.

\addtocounter{theorem}{-1}
\begin{theorem}[restated]

Let $\mathcal{S} \in \mathcal{P}(\mathcal{X}^n)$ be a set of inputs such that $\mathcal{S}$ contains an image that is not too dark; that is, $\exists\, \example \in \mathcal{S}$ for which $\exists\, x_i \text{ s.t. } x_i > \epsilon_1 / \epsilon_2$. Then

\begin{equation*}
    t_\text{combined}(\mathcal{S}) \;\supsetneq\;
    t_\text{add}(\mathcal{S}) \cup t_\text{\tm}(\mathcal{S})
    \qquad \textup{or equivalently} \qquad
    \exists \, \advexample \ \text{s.t.} \ 
    \begin{tabular}{ll}
        $\advexample \in t_\text{combined}(\mathcal{S})$ \\
        $\advexample \notin t_\text{add}(\mathcal{S}) \cup t_\text{\tm}(\mathcal{S})$
    \end{tabular}
\end{equation*}
\end{theorem}

\begin{proof}
\label{theorem_proof}

The above two statements are equivalent, so we focus on the formulation on the right. We calculate $\advexample$ and show that it satisfies the given criteria. Let $\example \in \mathcal{S}$ such that $\exists\, x_i \text{ s.t. } x_i > \epsilon_1 / \epsilon_2$. Without loss of generality, assume that in particular $x_2 > \epsilon_1 / \epsilon_2$. Then let

\begin{equation*}
    \advexample = ((1 - \epsilon_2) \: x_1 + \epsilon_1, (1 - \epsilon_2) \: x_2, \dots, (1 - \epsilon_2) \: x_n)
\end{equation*}

First, we show that $\advexample \in t_\text{combined}(\mathcal{S})$. Using the definition of $t_\text{combined}$ in (\ref{eq:combined_tm}), we set $c = 1 - \epsilon_2$, $\delta_1 = \epsilon_1$, and $\delta_2 = \dots = \delta_n = 0$, which generates $\advexample$. These values clearly satisfy the constraints in (\ref{eq:combined_tm}).

Second, we prove that $\advexample \notin t_\text{add}(\mathcal{S})$ by contradiction. Say that $\advexample \in t_\text{add}(\mathcal{S})$. Then $\exists\, \delta_1, \delta_2, \dots, \delta_n$ such that $\widetilde{x}_i = x_i + \delta_i$ and $\| \delta_i \| \leq \epsilon_1$. Consider $\delta_2$, which must satisfy $\widetilde{x}_2 = (1 - \epsilon_2) \: x_2 = x_2 + \delta_2$, or alternatively $\delta_2 = x_2 - (1 - \epsilon_2) \:  x_2 = \epsilon_2 \: x_2$. However, $x_2 > \epsilon_1 / \epsilon_2$ implies that $\delta_2 > \epsilon_1$, which is a contradiction since the constraints on $t_\text{add}$ specify that $| \delta_2 | \leq \epsilon_1$. Thus, $\advexample \notin t_\text{add}(\mathcal{S})$.

Third, we prove that $\advexample \notin t_\text{\tm}(\mathcal{S})$, again by contradiction. Say that $\advexample \in t_\text{\tm}(\mathcal{S})$. Then $\exists\, c \in [1 - \epsilon_2, 1 + \epsilon_2]$ such that $\widetilde{x}_i = c \: x_i$ for all $i$. Considering $i = 1, 2$, we have the following system of equations:

\begin{align*}
    \widetilde{x}_1 = c x_1 & = (1 - \epsilon_2) \: x_1 + \epsilon_1 \\
    \widetilde{x}_2 = c x_2 & = (1 - \epsilon_2) \: x_2
\end{align*}

From the second equation, we have $c = 1 - \epsilon_2$. However, using this in the first equation gives $(1 - \epsilon_2) \: x_1 = (1 - \epsilon_2) \: x_1 + \epsilon_1$, which implies $0 = \epsilon_1$. This is a contradiction since $\epsilon_1 > 0$, showing that $\advexample \notin t_\text{\tm}(\mathcal{S})$.
\end{proof}

\section{Experimental Setup}
\label{parameters}

We implement \modelname using the \texttt{mister\_ed} library \citep{jordan_quantifying_2019} and PyTorch \citep{paszke_automatic_2017}. Adversarial examples are generated by \ac{pgd} using the Adam optimizer \citep{kingma_adam:_2014} with learning rate 0.001. During adversarial training 100 iterations of Adam are used and during evaluation 300 iterations are used; see appendix \ref{attack_iter} for more information on the effect of the number of PGD iterations. After all iterations have completed, we choose the result of the iteration with the lowest loss as the adversarial example.

When combining attacks, we apply multiple attacks sequentially to the input example and optimize over the parameters of all attacks simultaneously, similarly to \citet{jordan_quantifying_2019}.

In all adversarial training experiments on CIFAR-10, we begin with a trained ResNet32 \citep{idelbayev_proper_2018} and then train it further on batches which are half original training data and half adversarial examples. We adversarially train with a batch size of 500 for 50 epochs. We preprocess images after adversarial perturbation, but before classification, by standardizing them based on the mean and standard deviation of each channel for all images in the dataset. The CIFAR-10 dataset can be obtained from \url{https://www.cs.toronto.edu/~kriz/cifar.html}.

In CIELUV color space (see section \ref{color_space}), we define
\begin{equation}
    (c_1, c_2, c_3) = \left(\frac{L}{100}, \frac{U + 100}{200}, \frac{V + 100}{200}\right)
\end{equation}
so that $(c_1, c_2, c_3) \in [0, 1]^3$.

For the experiments described in section \ref{lpips}, we use LPIPS v0.1 with AlexNet. 

\subsection{Regularization Parameters}
\label{regularization}

The objective function and constraints described in section \ref{generation} include a number of constants that can be used to regularize the outputs of the \modelname attack. Changing these constants alters the strength of the attack and the perceptual similarity of a generated adversarial example to the input.

First, $\epsilon_1$, $\epsilon_2$, and $\epsilon_3$ control the maximum amount by which a color in $\example$ can be changed to produce $\advexample$. For \ac{rgb} color space, we set $\epsilon_1 = \epsilon_2 = \epsilon_3 = 0.1$; that is, each channel of a color can change by up to $\sim 25/255$. This is greater than the usual $\epsilon = 8/255$ allowed for adversarial examples, but we find that the uniform perturbation used by the \tmname threat model allows each pixel to change by a greater amount while remaining almost indistinguishable. For the CIELUV color space, we let $\epsilon_1 = \epsilon_2 = \epsilon_3 = 0.06$. This corresponds to a maximum change of 6 in $L$ and a maximum change of 3 in $U$ and $V$, since we find that changes in luma are usually less noticeable than changes in chroma. The $\epsilon_i$ values for \ac{rgb} and CIELUV color spaces result in similar total amounts of perturbation, but the CIELUV color space allows the perturbation to be greater in areas where it is less noticeable.

Second, we can control the resolution of the grid $\mathcal{G}$ over which the perturbation function $\tmfunc(\cdot)$ is parameterized. Let $R_1 \times R_2 \times R_3$ be the resolution of $\mathcal{G}$. Lowering the resolution in a particular dimension acts as a regularizer because it allows less variation in how colors are transformed along that dimension. For \ac{rgb} color space, we use $R_1 = R_2 = R_3 = 25$. However, for CIELUV color space, we use $R_1 = 16$ and $R_2 = R_3 = 32$. With a high $R_1$ value, we find that the attack sometimes recolors different values of a particular hue very differently. For instance, the attack might make the light parts of a white car green and the dark parts purple. Lowering $R_1$ forces the attack to alter these colors more similarly.

Finally, $\lambda$ controls the importance of the smoothness optimization term $\mathcal{L}_\text{smooth}$. We always set $\lambda = 0.05$.

\section{Other Color Spaces}
\label{other_color_space}

Besides the \ac{rgb} and CIELUV color spaces, we also experimented with the \ac{hsv} and {YPbPr} color spaces. {HSV} presents difficulties when performing PGD because the derivative of the transformation from {RGB} is highly discontinuous; thus we use an approximation HSV$'$ which maps colors into a hexagonal pyramid instead of the standard HSV cone. A disadvantage of both {HSV} and {YPbPr} is that they were originally designed for transmitting video signals rather than as an accurate representation of how humans view colors.

Below, we present ReColorAdv using four color spaces; see C-\{LUV, RGB, YPbPr, HSV$'$\}. We also experiment with some additional variations:

\begin{itemize}
    \item We apply ReColorAdv separately to each channel, i.e. $(R, G, B) \rightarrow (f_1(R), f_2(G), f_3(B))$; see C-\{LUV, RGB\}-Sep-Channels.
    \item The human eye is more sensitive to variations in some colors (e.g. green) than others (e.g. blue). We experiment with separate bounds for each RGB channel based on this sensitivity ($\epsilon_R = 0.1, \epsilon_G = 0.05, \epsilon_B = 0.15$); see C-RGB-Sep-Bounds. Note that we already applied separate bounds in CIELUV color space, as detailed in appendix \ref{regularization}.
\end{itemize}

For each of these variations, the accuracy of an undefended model under the attack is shown.

\begin{center}
    \begin{tabularx}{\columnwidth}{@{}s@{\hskip 0.1in}b@{\hskip 0.1in}b@{\hskip 0.1in}b@{\hskip 0.1in}b@{\hskip 0.1in}b@{\hskip 0.1in}b@{\hskip 0.1in}b@{\hskip -0.2in}}
        \bf Original & \bf C-LUV & \bf C-LUV-Sep-Channels & \bf C-RGB & \bf C-RGB-Sep-Channels & \bf C-RGB-Sep-Bounds & \bf C-YPbPr & \bf C-HSV$'$ \\
        92.3\% & 4.4\% & 8.4\% & 8.2\% & 9.3\% & 8.4\% & 2.6\% & 2.1\%
    \end{tabularx}
\end{center}
\vspace*{-0.7\baselineskip}
\includegraphics[width=\columnwidth]{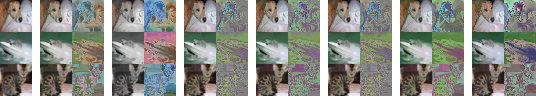}

\section{PGD Iterations and Attack Strength}
\label{attack_iter}

In our main results on CIFAR-10, we use 300 iterations of PGD to generate adversarial examples. Here, we also report results with using only 100 iterations. The combined attacks in particular are weaker when using fewer iterations, perhaps because they have more parameters that need to be optimized. This presents a challenge for efficient adversarial training since many inner loop iterations are needed to produce a strong attack.

\begin{center}
    \begin{tabular}{l|rrrrrrrrr}
        \toprule
        & \multicolumn{9}{c}{\textbf{Attack} (100 PGD iterations)} \\
        \bf Defense $\downarrow$ & \bf None & \bf C-RGB & \bf C & \bf D & \bf S & \bf C+S & \bf C+D & \bf S+D & \bf C+S+D \\
        \midrule
        \bf Undefended & 92.5 & 8.2 & 4.4 & \textbf{0.0} & 2.0 & 1.6 & \textbf{0.0} & \textbf{0.0} & \textbf{0.0} \\
        \bf C & 88.7 & 50.1 & 50.2 & 6.8 & 4.6 & 4.6 & 1.6 & \textbf{0.6} & 0.8 \\
        \bf D & 84.8 & 77.0 & 55.4 & 32.2 & 17.8 & 15.9 & 12.7 & 4.7 & \textbf{3.7} \\
        \bf S & 82.2 & 21.3 & 12.0 & 0.8 & 30.5 & 7.8 & \textbf{0.1} & 0.2 & \textbf{0.1} \\
        \bf C+S & 89.5 & 39.6 & 30.9 & 4.0 & 19.9 & 15.4 & \textbf{2.4} & 4.1 & 3.4 \\
        \bf C+D & 88.5 & 43.4 & 25.4 & 21.4 & 4.1 & \textbf{3.9} & 14.9 & 7.9 & 7.9 \\
        \bf S+D & 82.0 & 68.9 & 47.7 & 36.5 & 26.3 & 17.5 & 17.4 & 9.9 & \textbf{6.7} \\
        \bf C+S+D & 88.9 & 34.4 & 22.6 & 18.0 & 6.2 & \textbf{5.9} & 15.4 & 10.1 & 8.8 \\
        \midrule
        \bf TRADES & 84.2 & 81.6 & 64.8 & 53.8 & 30.9 & 23.2 & 29.0 & 11.6 & \textbf{8.1} \\
        \midrule
        \bf Undefended (B\&W) & 88.0 & 7.2 & 6.0 & \textbf{0.0} & 1.6 & 1.6 & \textbf{0.0} & \textbf{0.0} & \textbf{0.0} \\
        \bf C (B\&W) & 85.5 & 46.3 & 42.0 & 4.8 & 3.7 & 3.9 & 0.9 & \textbf{0.4} & 0.7 \\
        \bottomrule
    \end{tabular}
\end{center}

\section{Non-Additive Threat Models}
\label{other_tm}

Here, we discuss some other non-additive adversarial threat models that have been explored in the literature and how our work differs from them.

\paragraph{Spatial Threat Models} Some recent work has focused on \textit{spatial threat models}, which allow for slight perturbations of the locations of features in an input rather than perturbations of the features themselves. \citet{xiao_spatially_2018} propose StAdv, which optimizes the parameters of a smooth flow field that moves each pixel of an input image by a small, bounded distance to generate an example that fools the classifier. \citet{wong_wasserstein_2019} bound the Wasserstein distance between the original input and the adversarial example. \citet{engstrom_rotation_2017} apply an small rotation and translation to an input image to generate a misclassification.

\paragraph{Color Threat Models} 
Some previous work has explored uniform modification of an image's colors.
\citet{bhattad_big_2019} describe the cAdv attack, which converts an input image to black and white and then controls a colorization network to produce a differently-colored adversarial example.
\citet{zhang_limitations_2019} apply a single affine function to all pixels in an input image before performing PGD.
\citet{hosseini_semantic_2018} propose "Semantic Adversarial Examples," which allow modifications of the input image's hue and saturation. \citet{hosseini_limitation_2017} also explore inverting images to cause misclassification. These latter three papers can be considered as special examples of \tmname threat models. In the first, each channel undergoes an affine transformation, i.e. $\tmfunc(x_i) = \alpha x + \beta$. In the second, each pixel's hue and saturation is shifted by the same amount; that is, each pixel is transformed by the function $\tmfunc(h, s, v) = (h + \delta_h, s + \delta_s, v)$. In the third, each pixel is inverted, i.e. each pixel channel is transformed by the function $\tmfunc(x_i) = 1 - x_i$. However, the authors do not propose a general framework for these types of attacks, as we do. Furthermore, the adversarial examples generated by these attacks are often not realistic and/or not imperceptible. For example, their crafted adversarial examples include green skies, purple fields of grass, and inverted street signs---unlike our proposed \modelname attack, which results in imperceptible changes.

\paragraph{Other Threat Models} A few papers have focused on threat models that are neither additive, spatial, or color-based. \citet{zeng_adversarial_2019} perturb the properties of a 3D renderer (illumination, surface colors, materials, object placement) to render an image of an object which is unrecognizable to a classifier or other machine learning algorithm.
\citet{song_constructing_2018} uses a generative model to craft adversarial examples directly without perturbing other images. In contrast, we aim to augment the space of adversarial perturbations of images in the train or test sets.

\section{Additional Images}

\begin{figure}[H]
    \centering
    \includegraphics[width=\columnwidth]{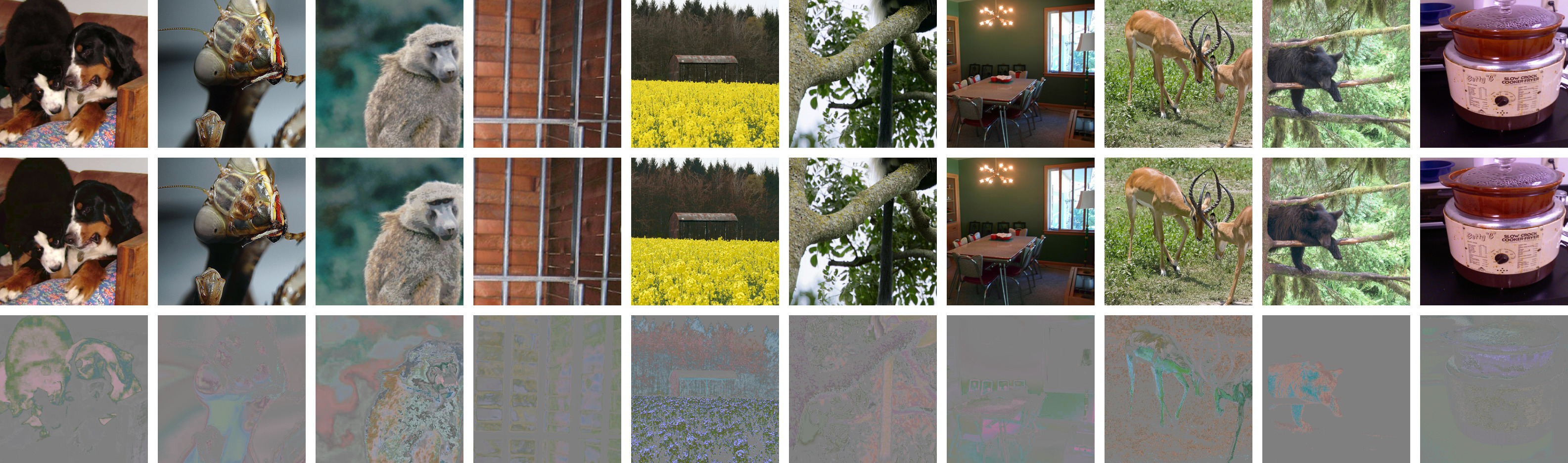}
    \caption{More adversarial examples like those in figure \ref{fig:examples_imagenet}, generated by \modelname against an Inception-v4 classifier on ImageNet. Top row: original images; middle row: adversarial examples; bottom row: magnified difference.
    }
\end{figure}

\begin{figure}[H]
    \centering
    \begin{minipage}[b]{0.08\columnwidth}
        \begin{flushright}
        Original \\[16.5pt]
        C \\[16.5pt]
        D \\[16.5pt]
        C+D \\[16.5pt]
        C+S+D\vspace{12pt}
        \end{flushright}
    \end{minipage}
    \quad
    \includegraphics[width=0.14\columnwidth]{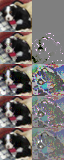}
    \includegraphics[width=0.14\columnwidth]{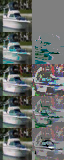}
    \includegraphics[width=0.14\columnwidth]{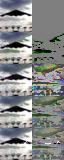}
    \includegraphics[width=0.14\columnwidth]{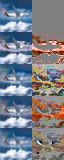}
    \includegraphics[width=0.14\columnwidth]{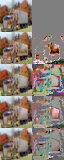}
    \includegraphics[width=0.14\columnwidth]{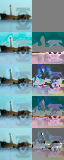}
    \caption{More adversarial examples like those in figure \ref{fig:examples_cifar}, generated with combinations of attacks against a CIFAR-10 WideResNet trained using TRADES. C is \modelname, D is delta attack, and S is StAdv attack \citep{xiao_spatially_2018}. The difference from the original is shown to the right of each example. Combinations of attacks tend to produce less perceptible changes than the attacks do separately.}
\end{figure}

\section{Lipschitz Regularization}

In addition to the regularizations defined in section \ref{reg_tm}, we can also enforce that the perturbation function $f(\cdot)$ in a functional threat model is Lipschitz for some suitably small $\kappa$:

\begin{equation} \label{eq:trans_lipschitz}
    \funcfam{lips} \triangleq \left\{ f: \mathcal{X} \rightarrow \mathcal{X} \mid \forall x_1, x_2 \in \mathcal{X} \  \|\tmfunc(x_1) - \tmfunc(x_2)\| \leq \kappa \|x_1 - x_2\| \right\}
\end{equation}

$\funcfam{lips}$ requires some smoothness in the perturbation function $\tmfunc(\cdot)$, ensuring that similar features in the input are mapped to similar features in the adversarial example. However, one disadvantage of $\funcfam{lips}$ is that it includes constant functions $f(x) = c$, i.e. functions which map every feature to a single value, removing salient features from the input. Thus, we ultimately use $\funcfam{smooth}$ instead. 

\end{document}